\documentclass[11pt]{article}

\usepackage{graphicx} 
\usepackage{caption}
\usepackage{subcaption}

\usepackage{authblk}
\usepackage{fullpage}

\usepackage[square]{natbib}

\usepackage{algorithm}
\usepackage{algorithmic}

\usepackage[outdir=./]{epstopdf}

\usepackage{hyperref}

\newcommand{\E}{\mathbb{E}}

\newcommand{\matt}[1]{}
\newcommand{\sj}[1]{}
\newcommand{\ssj}[1]{}
\newcommand{\ms}[1]{}

\newcommand{\abs}[1]{\left\lvert #1 \right\lvert}

\newcommand{\Is}{\mathcal I}

\newcommand{\Xs}{\mathcal X}
\newcommand{\Ys}{\mathcal Y}
\newcommand{\ep}{\varepsilon}

\newcommand{\reals}{\mathbb R}
\newcommand{\integers}{\mathbb Z}

\usepackage{amsmath,amsfonts}

\DeclareMathOperator{\diag}{\mathrm{diag}}
\DeclareMathOperator{\proj}{\mathrm{proj}}

\DeclareMathOperator{\Var}{\mathrm{Var}}
\DeclareMathOperator{\Ber}{\mathrm{Ber}}
\DeclareMathOperator{\betadist}{\mathrm{Beta}}
\DeclareMathOperator{\conv}{\mathrm{conv}}
\DeclareMathOperator{\boxset}{\mathrm{Box}}
\newcommand{\argmin}{\operatornamewithlimits{argmin}}
\newcommand{\argmax}{\operatornamewithlimits{argmax}}

\usepackage{amsthm}
\newtheorem{theorem}{Theorem}[section]
\newtheorem{lemma}{Lemma}[section]
\newtheorem{corollary}{Corollary}[section]
\newtheorem{conjecture}{Conjecture}[section]
\newtheorem{proposition}{Proposition}[section]
\theoremstyle{definition}

\newtheorem{remark}{Remark}[section]
\newtheorem{fact}{Fact}[section]

\title{Robust Budget Allocation via Continuous Submodular Functions}
\author{Matthew Staib}
\author{Stefanie Jegelka}
\affil{
    Computer Science and Artificial Intelligence Laboratory\\
    Massachusetts Institute of Technology\\
    \{mstaib, stefje\}@mit.edu
}
\date{}

\begin{document} 
\maketitle

\begin{abstract} 
  The optimal allocation of resources for maximizing influence, spread of information or coverage, has gained attention in the past years, in particular in machine learning and data mining. But in applications, the parameters of the problem are rarely known exactly, and using wrong parameters can lead to undesirable outcomes. We hence revisit a continuous version of the Budget Allocation or Bipartite Influence Maximization problem introduced by \citet{alon_optimizing_2012} from a robust optimization perspective, where an adversary may choose the least favorable parameters within a confidence set. The resulting problem is a nonconvex-concave saddle point problem (or game). We show that this nonconvex problem can be solved exactly by leveraging connections to continuous submodular functions, and by solving a constrained submodular minimization problem. Although constrained submodular minimization is hard in general, here, we establish conditions under which such a
  problem can be solved to arbitrary precision $\epsilon$.
\end{abstract}


\section{Introduction}
The optimal allocation of resources for maximizing influence, spread of information or coverage, has gained attention in the past few years, in particular in machine learning and data mining \citep{domingos2001mining,kempe_maximizing_2003,chen2009efficient,gomez2012influence,borgs_maximizing_2014}.

\matt{Needs motivation for why we care about the budget allocation problem? or is that taken care of by all the cites in the previous sentence?}
In the \emph{Budget Allocation Problem}, one is given a bipartite influence graph
between channels $S$ and people $T$, and the task is to assign a budget $y(s)$
to each channel $s$ in $S$ with the goal of maximizing the expected number of
influenced people $\Is(y)$. Each edge $(s,t) \in E$ between channel $s$ and
person $t$ is weighted with a probability $p_{st}$ that, e.g., an advertisement on radio station $s$ will influence
person $t$ to buy some product. The budget $y(s)$ controls how many independent
attempts are made via the channel $s$ to influence the people in $T$. The
probability that a customer $t$ is influenced when the advertising budget is
$y$ is 
\begin{equation}
    I_t(y) = 1 - \prod\nolimits_{(s,t) \in E} [1 - p_{st}]^{y(s)},
\end{equation}
and hence the expected number of influenced people is $\Is(y) = \sum_{t\in T}
I_t(y)$. We write $\Is(y;p)=\Is(y)$ to make the dependence on the
probabilities $p_{st}$ explicit. The total budget $y$ must remain within some
feasible set $\Ys$ which may encode e.g. a total budget limit $\sum_{s \in
S}y(s) \leq C$. We allow the budgets $y$ to be continuous, as in
\citep{bian_guaranteed_2016}.

Since its introduction by \citet{alon_optimizing_2012}, several works have
extended the formulation of Budget Allocation and provided
algorithms
\cite{bian_guaranteed_2016,hatano_lagrangian_2015,maehara_budget_2015,soma_optimal_2014,soma_generalization_2015}.
Budget Allocation may also be viewed as influence maximization on
a bipartite graph, where information spreads as in the Independent Cascade
model. 
For integer $y$,
Budget
Allocation and Influence Maximization are NP-hard. Yet, constant-factor approximations 
are possible, and build on the fact that the
influence function is submodular in the binary case, and
\emph{DR-submodular} in the integer case
\citep{soma_optimal_2014,hatano_lagrangian_2015}.
If $y$ is continuous, the problem is a concave maximization problem.

The formulation of Budget Allocation assumes that the transmission
probabilities are known exactly. But this is rarely true in practice. Typically, the probabilities $p_{st}$, and possibly the graph itself, must be inferred from observations 
%
\citep{gomez_inferring_2010,du_scalable_2013,narasimhan_learnability_2015,du2014influence,netrapalli_learning_2012}. In Section~\ref{sec:expts} we will see that a misspecification or point estimate of parameters $p_{st}$ can lead to much reduced outcomes.
A more realistic assumption is to know 
\emph{confidence intervals} for the $p_{st}$. Realizing this severe deficiency, recent work
studied robust versions of Influence Maximization, where a budget $y$ must be chosen
that maximizes the worst-case approximation ratio over a set of possible
influence functions
\citep{he_robust_2016,chen_robust_2016,lowalekar_robust_2016-1}. \matt{technically the lowalekar paper does some slightly different regret thing but is perhaps close enough} The resulting
optimization problem is hard but admits bicriteria approximations.

In this work, we revisit Budget Allocation under uncertainty from the perspective of robust optimization \citep{bertsimas_theory_2011,bental09_book}. We maximize the
worst-case influence -- not approximation ratio -- for $p$ in a confidence set
centered around the ``best guess'' (e.g., posterior mean). 
This avoids pitfalls of the 
approximation ratio formulation (which can be misled to return poor worst-case budgets, as demonstrated in Appendix~\ref{sec:appendix-ratio}), while also allowing us to formulate the problem as a max-min game:
\begin{equation}
  \label{eq:minmaxp}
  \max_{y \in \mathcal{Y}}\; \min_{p \in \mathcal{P}}\; \Is(y;p),
\end{equation}
where an ``adversary'' can arbitrarily manipulate $p$ within the confidence set $\mathcal{P}$.
With $p$ fixed, $\Is(y;p)$ is concave in $y$. However, the influence function $\Is(y;p)$ is not convex, and not even quasiconvex, in the adversary's variables $p_{st}$.

The new, key insight we exploit in this work is that $\Is(y;p)$ has the property of \emph{continuous submodularity} in $p$ -- in contrast to previously exploited submodular maximization in $y$ -- and can hence be minimized by generalizing techniques from discrete submodular optimization \cite{bach_submodular_2015}. The techniques in \citep{bach_submodular_2015}, however, are restricted to box constraints, and do not directly apply to our confidence sets. In fact, general constrained submodular minimization is hard \citep{svitkina_submodular_2011,goel_approximability_2009,iwata_submodular_2009}.
We make the following contributions:
\vspace{-3pt}
\begin{enumerate}\setlength{\itemsep}{0pt}
\item We present an algorithm with optimality bounds for Robust Budget Allocation in the nonconvex adversarial scenario \eqref{eq:minmaxp}.
\item We provide the first results for continuous submodular minimization with box constraints and one more ``nice'' constraint, and conditions under which the algorithm is guaranteed to return a global optimum. 
\end{enumerate}

\subsection{Background and Related Work}
We begin with some background material and, along the way, discuss related work.

\subsubsection{Submodularity over the integer lattice and continuous domains}
Submodularity is perhaps best known as a property of set functions. A function $F: 2^V \to \mathbb{R}$ defined on subsets $S \subseteq V$ of a ground set $V$ is \emph{submodular} if for all sets $S,T \subseteq V$, it holds that
$F(S) + F(T) \geq F(S \cap T) + F(S \cup T)$. 
A similar definition extends to functions defined over a distributive lattice $\mathcal{L}$, e.g. the integer lattice. Such a function $f$ is submodular if for all $x, y \in \mathcal{L}$, it holds that 
\begin{equation}
  f(x) + f(y) \geq f(x \vee y) + f(x \wedge y). \label{eq:subm_lattice}
\end{equation}
For the integer lattice and vectors $x,y$, $x \vee y$ denotes the coordinate-wise maximum and $x \wedge y$ the coordinate-wise minimum. Submodularity has also been considered on continuous domains $\mathcal{X} \subset \mathbb{R}^d$, where, if $f$ is also twice-differentiable, the property of submodularity means that all off-diagonal entries of the the Hessian are nonpositive, i.e., $\frac{\partial f(x)}{\partial x_i \partial x_j} \leq 0$ for all $i \not = j$ \citep[Theorem 3.2]{topkis1978minimizing}. These functions may be convex, concave, or neither.

Submodular functions on lattices
can be minimized by a reduction to set functions, more precisely, ring families \cite{birkhoff37}. Combinatorial algorithms for submodular optimization on lattices are discussed in~\citep{khachaturov2012supermodular}. More recently, \citet{bach_submodular_2015} extended results based on the convex Lov\'asz extension, by building on connections to optimal transport. The subclass of $L^\natural$-convex functions admits strongly polynomial time minimization \citep{murota2003discrete,kolmogorov09,murota14}, but does not apply in our setting. 

Similarly, results for submodular maximization extend to integer lattices, e.g. \citep{gottschalk15}. Stronger results are possible if the submodular function also satisfies \emph{diminishing returns}: for all $x \leq y$ (coordinate-wise) and $i$ such that $y+e_i \in \mathcal{X}$, it holds that $f(x+e_i) - f(x) \geq f(y+e_i)-f(y)$.
For such \emph{DR-submodular} functions, many approximation results for the set function case extend \citep{bian_guaranteed_2016,soma_generalization_2015,soma_optimal_2014}. In particular, \citet{ene_reduction_2016} show a generic reduction to set function optimization that they apply to maximization. In fact, it also applies to minimization:


\begin{proposition}
\label{prop:dr-min}
    A DR-submodular function $f$ defined on
    $\prod_{i=1}^n [k_i]$ can be minimized in strongly polynomial time 
    $O(n^4 \log^4 k \cdot \log^2 (n \log k) \cdot EO + n^4 \log^4 k \cdot \log^{O(1)} (n \log k) )$,
    where $k = \max_i k_i$ and 
    $EO$ is the time complexity of evaluating $f$.
    Here, $[k_i] = \{0,1,\dots,k_i-1\}$.
\end{proposition}
\begin{proof}
    The function $f$ can be reduced to a submodular set function $g : 2^V \to
    \reals$ via \citep{ene_reduction_2016}, where $\abs V = O(n \log k)$. The
    function $g$ can be evaluated via mapping from $2^V$ to the domain of $f$,
    and then evaluating $f$, in time $O(n \log k \cdot EO)$. We can directly
    substitute these complexities into the runtime bound from
    \citep{lee_faster_2015}.
\end{proof}
In particular, the time complexity is logarithmic in $k$. For general lattice submodular functions, this is not possible without further assumptions.

\subsubsection{Related Problems}
%
A sister problem of Budget Allocation is \emph{Influence Maximization} on general graphs, where a set of seed nodes is selected to start a propagation process. The influence function is still monotone submodular and amenable to the greedy algorithm \cite{kempe_maximizing_2003}, but it cannot be evaluated explicitly and requires approximation \cite{chen_scalable_2010}.
%
%
\emph{Stochastic Coverage} \citep{goemans_stochastic_2006-1} is a version of Set Cover where the covering sets $S_i \subset V$ are random. A variant of Budget Allocation can be written as stochastic coverage with multiplicity. Stochastic Coverage has mainly been studied in the online or adaptive setting, where logarithmic approximation factors can be achieved \citep{golovin_adaptive_2011,deshpande_approximation_2016-1,adamczyk_submodular_2016}.

Our objective function~\eqref{eq:minmaxp} is a \emph{signomial} in $p$, i.e., a linear combination of monomials of the form $\prod_i x_i^{c_i}$. General signomial optimization is NP-hard \citep{chiang_geometric_2005}, but certain subclasses are tractable: \emph{posynomials} with all nonnegative coefficients can be minimized via Geometric Programming \citep{boyd_tutorial_2007}, and signomials with a single negative coefficient admit sum of squares-like relaxations \citep{chandrasekaran_relative_2016}. Our problem, a constrained posynomial maximization, is not in general a geometric program. Some work addresses this setting via monomial approximation \citep{pascual_constrained_1970,ecker_geometric_1980}, but, to our knowledge, our algorithm is the first that solves this problem to arbitrary accuracy.

\subsubsection{Robust Optimization}
Two prominent strategies of addressing uncertainty in parameters of optimization problems are stochastic and robust optimization. If the distribution of the parameters is known (stochastic optimization), formulations such as value-at-risk (VaR) and conditional value-at-risk (CVaR) \cite{rockafellar_optimization_2000,rockafellar_conditional_2002} apply. In contrast, robust optimization \citep{bental09_book,bertsimas_theory_2011} assumes that the parameters (of the cost function and constraints) can vary arbitrarily within a known confidence set $U$, and the aim is to optimize the worst-case setting, i.e., 
\begin{equation}
\min_y \sup_{u,A,b \in \mathcal{U}}\{g(y;u) \; \text{s.t. } Ay \leq b\}.
\end{equation}
Here, we will only have uncertainty in the cost function.

In this paper we are principally concerned with robust maximization of the continuous influence function $\Is(y)$, but mention some results for the discrete case.
While there exist results for robust and CVaR optimization of modular (linear) functions \citep{nikolova_approximation_2010,bertsimas2003robust}, submodular objectives do not in general admit such optimization \cite{maehara_risk_2015}, but variants admit approximations \citep{zhang_minimizing_2014}. The brittleness of submodular optimization under noise has been studied in \citep{balkanski2016power,balkanski2015limitations,hassidim2016submodular}.

Approximations for robust submodular and influence optimization have been studied in \citep{krause_robust_2008,he_robust_2016,chen_robust_2016,lowalekar_robust_2016-1}, where an adversary can pick among a \emph{finite} set of objective functions or remove selected elements \cite{orlin16_robust}.

\section{Robust and Stochastic Budget Allocation}
\label{sec:model}



The unknown parameters in Budget Allocation 
are the transmission probabilities $p_{st}$ or edge weights in a graph. If these are estimated from data, we may have posterior distributions or, a weaker assumption, confidence sets for the parameters.
For ease of notation, we will work with the failure probabilities $x_{st}=1-p_{st}$ instead of the $p_{st}$ directly, and write $\Is(y;x)$ instead of $\Is(y;p)$.

\subsection{Stochastic Optimization} 
\label{sec:stoch-opt}
If a (posterior) distribution of the parameters is known, a simple strategy is to use expectations. 
%
We place a uniform prior on $x_{st}$, and observe $n_{st}$ independent observations drawn from $\Ber(x_{st})$. If we observe $\alpha_{st}$ failures and and $\beta_{st}$ successes, the resulting posterior distribution on the variable $X_{st}$ is $\betadist(1+\alpha_{st},1+\beta_{st})$.
Given such a posterior, we may optimize
\begin{align}
  \label{eq:stoch1}
  \max_{y\in\Ys} &\,\;\Is(y; \E[X]), \text{ or }\\
  \label{eq:stoch2}
  \max_{y\in\Ys} &\,\; \E[\Is(y; X)].
\end{align}
\begin{proposition}
  Problems \eqref{eq:stoch1} and \eqref{eq:stoch2} are concave maximization problems over the (convex) set $\Ys$ and can be solved exactly.
\end{proposition}
Concavity of \eqref{eq:stoch2} follows since it is an expectation over concave functions, and the problem
can be solved by stochastic gradient ascent or by explicitly computing gradients. 


Merely maximizing expectation does not explicitly account for volatility and hence risk. One option is to include variance \cite{ben-tal_robust_2000,bertsimas_theory_2011,
  atamturk_polymatroids_2008}:
\begin{equation}
  \label{eq:mean-plus-var}
  \min_{y\in\Ys}\, -\E[\Is(y; X)] + \ep \sqrt{\Var(\Is(y; X))},
\end{equation}
but in our case this CVaR formulation seems difficult:
\begin{fact}
    \label{fact:variance-bad}
    For $y$ in the nonnegative orthant, the term 
    $\sqrt{\Var(\Is(y; X))}$ need not be convex or concave, and need not be
    submodular or supermodular.
\end{fact}
This observation does not rule out a solution, but the apparent difficulties further motivate a robust formulation that, as we will see, is amenable to optimization.

\subsection{Robust Optimization}\label{subsec:robust}
The focus of this work is the robust version of Budget Allocation, where we allow an adversary to arbitrarily set the parameters $x$ within an uncertainty set $\Xs$. This uncertainty set may result, for instance, from a known distribution, or simply assumed bounds. Formally, we solve
\begin{equation}
    \label{eq:minmax}
    \max_{y \in \Ys} \min_{x \in \Xs} \;\Is(y; x),
\end{equation}  
where $\Ys \subset \reals^S_+$ is a convex set with an efficient projection
oracle, and $\Xs$ is an uncertainty set containing an estimate $\hat x$. In the sequel, we use uncertainty sets $\Xs = \left\{ x \in \boxset(l,u) : R(x)
 \leq B \right\}$, where $R$ is a distance (or divergence) from the estimate
$\hat x$, and $\boxset(l,u)$ is the box $\prod_{(s,t)\in E}
[l_{st},u_{st}]$. The intervals $[l_{st},u_{st}]$ can be thought of as either
confidence intervals around $\hat x$, or, if $[l_{st},u_{st}]=[0,1]$, enforce that 
each $x_{st}$ is a valid probability.

Common examples of uncertainty sets used in Robust Optimization are \emph{Ellipsoidal} and \emph{D-norm uncertainty sets} \cite{bertsimas_theory_2011}. Our algorithm in Section~\ref{sec:ccsfm} applies to both.

\textbf{Ellipsoidal uncertainty.}
The ellipsoidal or quadratic uncertainty set is defined by
\begin{equation*}
    \Xs^Q(\gamma) = \{ x \in \boxset(0,1) : (x - \hat x)^T \Sigma^{-1} (x - \hat x) \leq \gamma \},
\end{equation*}
where $\Sigma$ is the covariance of the random vector $X$ of probabilities
distributed according to our Beta posteriors. In our case,
since the distributions on each $x_{st}$ are independent, $\Sigma^{-1}$ is 
actually diagonal. Writing $\Sigma = \diag(\sigma^2)$, we have
\begin{equation*}
    \Xs^Q(\gamma) = \Big\{ x \in \boxset(0,1) : \sum_{(s,t) \in E} R_{st}(x_{st})\leq \gamma \Big\},
\end{equation*}
where $R_{st}(x) = (x_{st} - \hat x_{st})^2 \sigma_{st}^{-2}$.

\textbf{D-norm uncertainty.}
The D-norm uncertainty set is similar to an $\ell_1$-ball around $\hat x$, and is defined as
\begin{align*}
    \Xs^D(\gamma) =& \Big\{ x : \exists c\in \boxset(0,1) \; \mathrm{s.t.} \vphantom{\sum_{(s,t) \in E}} \Big. \\[-3pt] 
    &x_{st} = \hat x_{st} + (u_{st} - \hat x_{st}) c_{st}, \;
     \Big. \sum_{(s,t)\in E} c_{st} \leq \gamma \Big\}.
\end{align*}
Essentially, we allow an adversary to increase $\hat x_{st}$ up to some upper
bound $u_{st}$, 
subject to some total budget $\gamma$ across
all terms $x_{st}$.
The set $\Xs^D(\gamma)$ can be rewritten as
\begin{equation*}
    \Xs^D(\gamma) = \Big\{ x \in \boxset(\hat x, u) : \sum_{(s,t) \in E} R_{st}(x_{st}) \leq \gamma \Big\},
\end{equation*}
where $R_{st}(x_{st}) = (x_{st} - \hat x_{st}) / (u_{st} - \hat x_{st})$ is
the fraction of the interval $[\hat x_{st}, u_{st}]$ we have used up in
increasing $x_{st}$.

The min-max formulation $\max_{y\in\Ys} \min_{x\in\Xs} \Is(y;x)$ has several
merits: the model is not tied to a specific learning algorithm for the
probabilities $x$ as long as we can choose a suitable confidence set.
Moreover, this formulation allows to fully hedge against a worst-case scenario.


\section{Optimization Algorithm}

As noted above, 
the function $\Is(y;x)$ is concave as a function of $y$ for fixed $x$. As a pointwise minimum of concave functions, $F(y) := \min_{x \in \Xs} \Is(y;x)$ is concave. Hence, if we can compute subgradients of $F(y)$, we can solve our max-min-problem via the subgradient method, as outlined in Algorithm \ref{alg:subgrad}.

A subgradient $g_y \in \partial F(y)$ at $y$ is given by the gradient of $\Is(y;x^*)$ for the minimizing $x^* \in \arg\min_{x\in\Xs} \Is(y;x)$, i.e., $g_y = \nabla_y \Is(y;x^*)$. Hence, we must be able to compute $x^*$ for any $y$.
We also obtain a duality gap: for any $x',y'$ we have
\begin{equation}
    \label{eq:duality-gap}
    \min_{x\in\Xs} \Is(y'; x) 
    \leq \max_{y \in \Ys} \min_{x\in\Xs} \Is(y; x) 
    \leq \max_{y \in \Ys} \Is(y; x').
\end{equation}
This means we can estimate the optimal value $\Is^*$ and
use it in Polyak's
stepsize rule for the subgradient method
\cite{polyak_introduction_1987}.

\begin{algorithm}[tb]
    \caption{Subgradient Ascent}
    \label{alg:subgrad}
\begin{algorithmic}
    \STATE {\bfseries Input:} suboptimality tolerance $\ep > 0$, initial feasible budget $y^{(0)} \in \Ys$
    \STATE {\bfseries Output:} $\ep$-optimal budget $y$ for Problem~\eqref{eq:minmax}
    \REPEAT
    \STATE $x^{(k)} \leftarrow \arg\min_{x \in \Xs} \Is(y^{(k)}; x)$
    \STATE $g^{(k)} \leftarrow \nabla_y \Is(y^{(k)}; x^{(k)})$
    \STATE $L^{(k)} \leftarrow \Is(y^{(k)}; x^{(k)})$
    \STATE $U^{(k)} \leftarrow \max_{y \in \Ys} \Is(y; x^{(k)})$
    \STATE $\gamma^{(k)} \leftarrow (U^{(k)} - L^{(k)}) / \lVert g^{(k)} \rVert^2_2$
    \STATE $y^{(k+1)} \leftarrow \proj_\Ys ( y^{(k)} + \gamma^{(k)} g^{(k)} )$     \STATE $k \leftarrow k+1$
    \UNTIL{$U^{(k)} - L^{(k)} \leq \ep$}
\end{algorithmic}
\end{algorithm}

But $\Is(y;x)$ is not convex in $x$, and not even quasiconvex. For example, standard
methods \citep[Chapter 12]{wainwright_fundamental_2004} imply that $f(x_1,x_2,x_3) = 1-x_1x_2 -\sqrt{x_3}$ is not quasiconvex on $\reals^3_+$. 
Moreover, the above-mentioned signomial optimization techniques do not apply for an exact solution either.
So, it is not immediately clear that we can solve the inner optimization problem.

The key insight we will be using is that $\Is(y;x)$ has a different beneficial property: while not convex, $\Is(y;x)$ as a function of $x$ is \emph{continuous submodular}.
\begin{lemma}
    \label{lem:submod-product}
    Suppose we have $n \geq 1$ differentiable functions $f_i : \reals \to
    \reals_+ $, for $i=1,\dots,n$, either all nonincreasing or all
    nondecreasing. Then, $f(x) = \prod_{i=1}^n f_i(x_i)$ is a continuous
    supermodular function from $\reals^n$ to $\reals_+ $. 
\end{lemma}
\begin{proof}
  For $n=1$, the resulting function
is modular and therefore supermodular.
In the case $n\geq2$,
we simply need to compute derivatives. The mixed derivatives are 
\begin{equation}
    \frac{\partial f}{\partial x_i \partial x_j}
    = f_i^\prime(x_i) f_j^\prime(x_j) \cdot \prod_{k \not = i,j} f_k(x_k).
\end{equation}
By monotonicity, $f_i^\prime$ and $f_j^\prime$ have the same sign, so their
product is nonnegative, and since each $f_k$ is nonnegative, the entire
expression is nonnegative. Hence, $f(x)$ is continuous supermodular by Theorem 3.2 of~\citep{topkis1978minimizing}.
\end{proof}

\begin{corollary}
    \label{corollary:submodular}
    The influence function $\Is(y;x)$ defined in Section \ref{sec:model} is continuous
    submodular in $x$ over the nonnegative orthant, for each $y \geq 0$.
\end{corollary}
\begin{proof}
    Since submodularity is preserved under summation, it suffices to show
    that each function $I_t(y)$ is continuous submodular. By Lemma \ref{lem:submod-product},
    since $f_s(z) = z^{y(s)}$ is nonnegative and monotone nondecreasing for $y(s) \geq 0$,
    the product $\prod_{(s,t) \in E} x_{st}^{y(s)}$ is continuous supermodular in $x$.
    Flipping the sign and adding a constant term yields $I_t(y)$, which is hence 
    continuous submodular. 
\end{proof}

\begin{conjecture}
    \label{conjecture:strong-duality}
    Strong duality holds, i.e. 
    \begin{equation}
        \max_{y\in\Ys} \min_{x\in\Xs} \Is(y;x)
        = \min_{x\in\Xs}\max_{y\in\Ys}  \Is(y;x).
    \end{equation}
\end{conjecture}
If strong duality holds, then the duality gap $\max_{y\in\Ys} \Is(y;x^*) -
\min_{x\in\Xs} \Is(y^*;x)$ in Equation~\eqref{eq:duality-gap} is zero at
optimality. 
If $\Is(y;x)$ were
quasiconvex in $x$, strong duality would hold by Sion's min-max theorem, but this is not the case.
In practice, we observe that the duality gap always
converges to zero. 

\citet{bach_submodular_2015} demonstrates how to
minimize a continuous submodular function $H(x)$ subject to box
constraints $x \in \boxset(l,u)$, up to an arbitrary suboptimality gap
$\ep>0$. The constraint set $\Xs$ in our Robust Budget Allocation problem,
however, has box constraints with an additional constraint $R(x)
\leq B$. This case is not addressed in any previous work.
Fortunately, for a large class of 
functions $R$, there is still an efficient algorithm for continuous submodular
minimization, which we present in the next section.

\subsection{Constrained Continuous Submodular Function Minimization}
\label{sec:ccsfm}
We next address an algorithm for minimizing a monotone continuous submodular function $H(x)$ subject to box constraints $x \in \boxset(l,u)$ and a constraint $R(x) \leq B$:
\begin{equation}
    \label{eq:constrained-prob}
    \begin{array}{ll}
        \text{minimize} & H(x) \\
        \text{s.t.} & R(x) \leq B \\
        & x \in \boxset(l,u).
    \end{array}
\end{equation}
If $H$ and $R$ were convex, the constrained problem would be 
equivalent to solving, with the right Lagrange multipler $\lambda^* \geq 0$:
\begin{equation}
    \label{eq:regularized-prob}
    \begin{array}{ll}
        \text{minimize} & H(x) + \lambda^* R(x) \\
        \text{s.t.} & x \in \boxset(l,u).
    \end{array}
\end{equation}
Although $H$ and $R$ are not necessarily convex here, it turns out that a similar approach indeed applies. The main idea of our approach bears similarity with \citep{nagano2011size} for the set function case, but our setting with continuous functions and various uncertainty sets is more general, and requires more argumentation. We outline our theoretical results here, and defer further implementation details and proofs to the appendix.

Following \cite{bach_submodular_2015}, we discretize the problem; for a sufficiently fine discretization, we will achieve arbitrary accuracy. Let $A$ be an \emph{interpolation mapping} that maps the
discrete set $\prod_{i=1}^n [k_i]$ into $\boxset(l,u) = \prod_{i=1}^n [l_i,u_i]$ via
the componentwise interpolation functions $A_i : [k_i] \to [l_i,u_i]$. We say
$A_i$ is \emph{$\delta$-fine} if $A_i(x_i + 1) - A_i(x_i) \leq \delta$ for all $x_i
\in \{0, 1, \dots, k_i-2\}$, and we say the full interpolation function $A$ is
$\delta$-fine if each $A_i$ is $\delta$-fine. 

This mapping yields functions $H^\delta : \prod_{i=1}^n [k_i] \to \reals$ and $R^\delta : \prod_{i=1}^n [k_i] \to \reals$ via $H^\delta(x) = H(A(x))$ and $R^\delta(x) =
R(A(x))$. $H^\delta$ is lattice submodular (on the integer lattice).
This construction leads to a reduction of Problem~\eqref{eq:constrained-prob} to a submodular minimization problem over the integer lattice:
\begin{equation}
    \label{eq:regularized-prob-discrete}
    \begin{array}{ll}
        \text{minimize} & H^\delta(x) + \lambda R^\delta(x) \\
        \text{s.t.} & x \in \prod_{i=1}^n [k_i].
    \end{array}
\end{equation}
Ideally, there should then exist a $\lambda$ such that the associated minimizer $x(\lambda)$ yields a close to optimal solution for the constrained problem. Theorem~\ref{thm:constrained-opt} below states that this is indeed the case.

Moreover, a second benefit of submodularity is that we can find the entire solution path for Problem \eqref{eq:regularized-prob-discrete} by solving a single optimization problem.
\begin{lemma}
    \label{lemma:prob-equiv}
    Suppose $H$ is continuous submodular, and suppose the regularizer $R$ is
    strictly increasing and separable: $R(x) = \sum_{i=1}^n R_i(x_i)$. Then
    we can recover a minimizer $x(\lambda)$ for the induced discrete Problem
    \eqref{eq:regularized-prob-discrete} for any $\lambda \in \reals$ by
    solving a single convex optimization problem.
\end{lemma}
The problem in question arises from a relaxation $h_\downarrow$ that extends $H^\delta$ in each coordinate $i$ to a function on distributions over the domain $[k_i]$. These distributions are represented via their inverse cumulative distribution functions $\rho_i$, which take the coordinate $x_i$ as input, and output the probability of exceeding $x_i$. The function $h_\downarrow$ is an analogue of the \emph{Lov\'asz extension} of set functions to continuous submodular functions \citep{bach_submodular_2015}, it is convex and coincides with $H^\delta$ on lattice points. 

Formally, this resulting single optimization problem is:
\begin{equation}
    \label{eq:rho-prob}
    \begin{array}{ll}
        \text{minimize} & h_\downarrow(\rho) + \sum_{i=1}^n \sum_{j_i=1}^{k_i-1} a_{i x_i}(\rho_{i}(x_i)) \\
        \text{s.t.} & \rho \in \prod_{i=1}^n \reals_\downarrow^{k_i - 1}
    \end{array}
\end{equation}
where $\reals_\downarrow^k$ refers to the set of ordered vectors $z \in \reals^k$ that satisfy $z_1 \geq z_2 \geq \dots \geq z_k$,
the notation $\rho_i(x_i)$ denotes the $x_i$-th coordinate of the vector $\rho_i$,
and the $a_{i x_i}$ are strictly convex functions given by
\begin{equation}
  a_{i x_i}(t) = \frac12 t^2 \cdot [R_i^\delta(x_i) - R_i^\delta(x_i-1)].
\end{equation}

Problem~\eqref{eq:rho-prob} can be solved by Frank-Wolfe methods \citep{frank_algorithm_1956,dunn1978conditional,lacoste-julien_convergence_2016,jaggi_revisiting_2013}.
This is because the greedy algorithm for computing subgradients of the Lov\'asz extension can be generalized, and yields a linear optimization oracle for the dual of Problem~\eqref{eq:rho-prob}.
We detail the relationship between Problems~\eqref{eq:regularized-prob-discrete} and \eqref{eq:rho-prob}, as well as how to implement the Frank-Wolfe methods,
in Appendix~\ref{sec:appendix-ccsfm}.

Let $\rho^*$ be the optimal solution for Problem~\eqref{eq:rho-prob}. For any $\lambda$, we obtain a rounded solution $x(\lambda)$ for Problem~\eqref{eq:regularized-prob-discrete} 
by thresholding: we set $x(\lambda)_i = \max\{j \mid 1 \leq j \leq k_i-1, \; \rho^*_{i}(j) \geq \lambda\}$, or zero if $\rho^*_i(j) < \lambda$ for all $j$. 
Each $x(\lambda')$ is the optimal solution for Problem~\eqref{eq:regularized-prob-discrete} with $\lambda=\lambda'$.
We use the largest parameterized solution $x(\lambda)$ that is still feasible, i.e. the solution $x(\lambda^*)$ where $\lambda^*$ solves
\begin{equation}
    \label{eq:optlambda}
    \begin{array}{ll}
        \text{min} & H^\delta(x(\lambda)) \\
        \text{s.t.} & \lambda \geq 0 \\
        & R^\delta(x(\lambda)) \leq B.
    \end{array}
\end{equation}
This $\lambda^*$ can be found efficiently via binary search or a linear scan.
\begin{theorem}
  \label{thm:constrained-opt}
  Let $H$ be continuous submodular and monotone decreasing, with $\ell_\infty$-Lipschitz constant $G$, and let $R$ be strictly increasing and separable. Assume all entries $\rho^*_i(j)$ of the optimal solution $\rho^*$ of Problem \eqref{eq:rho-prob} are distinct. Let $x' = A(x(\lambda^*))$ be the thresholding corresponding to the optimal solution
  $\lambda^*$ of Problem \eqref{eq:optlambda}, mapped back into the original
  continuous domain $\Xs$. Then $x'$ is feasible for the continuous Problem
  \eqref{eq:constrained-prob}, and is a $2G\delta$-approximate solution:
  \begin{equation*}
    H(x') \leq 2G \delta + \min_{x \in \boxset(l,u), \; R(x) \leq B} H(x).
  \end{equation*}
\end{theorem}
Theorem \ref{thm:constrained-opt} implies an algorithm for solving
Problem~\eqref{eq:constrained-prob} to $\ep$-optimality: (1) set $\delta =
\ep/G$, (2) compute $\rho^*$ which solves Problem~\eqref{eq:rho-prob}, (3)
find the optimal thresholding of $\rho^*$ by determining the
smallest $\lambda^*$ for which $R^\delta(x(\lambda^*)) \leq B$, and (4)  map
$x(\lambda^*)$ back into continuous space via the interpolation mapping $A$. 

\paragraph{Optimality Bounds.}
\label{rem:soln-dependent-bounds}
    Theorem~\ref{thm:constrained-opt} is proved by comparing $x'$ and $x^*$
    to the optimal solution on the discretized mesh 
    \[
        x_d^* \in \argmin_{x \in \prod_{i=1}^n [k_i] : R^\delta(x) \leq B} H^\delta(x).
    \]
    Beyond the theoretical guarantee of Theorem~\ref{thm:constrained-opt},
    for any problem instance and candidate solution $x'$, we can compute
    a bound on the gap between $H(x')$ and $H^\delta(x_d^*)$. The following two bounds are proved in the appendix:
    \vspace{-3pt}
    \begin{enumerate}\setlength{\itemsep}{1pt}
        \item We can generate a discrete point $x(\lambda_+)$ satisfying
        \begin{equation*}
            H(x') \leq [ H(x') - H^\delta(x(\lambda_+))] + H^\delta(x_d^*).
        \end{equation*}
        \item The Lagrangian yields the bound
        \[
            H(x') \leq \lambda^* (B - R(x')) + H^\delta(x_d^*).
        \]
    \end{enumerate}

\paragraph{Improvements.}
\label{rem:rho-distinct}
The requirement in Theorem~\ref{thm:constrained-opt} that the elements of $\rho^*$ be distinct may seem somewhat restrictive, but as long as $\rho^*$ has distinct elements in the neighborhood of our particular $\lambda^*$, this bound still holds. We see in Section~\ref{sec:synthetic-optimality} that in practice,
$\rho^*$ almost always has distinct elements in the regime we care about, and the bounds of Remark~\ref{rem:soln-dependent-bounds} are very good.

If $H$ is DR-submodular and $R$ is affine in each coordinate, then Problem~\eqref{eq:regularized-prob-discrete} can be represented more compactly via the reduction of \citet{ene_reduction_2016}, and hence problem~\eqref{eq:constrained-prob} can be solved more efficiently. In particular, the influence function $\Is(y;x)$ is DR-submodular in $x$ when for each $s$, $y(s)=0$ or $y(s) \geq 1$.

\subsection{Application to Robust Budget Allocation}
The above algorithm directly applies to Robust Allocation with the uncertainty sets in Section~\ref{subsec:robust}. The ellipsoidal uncertainty set $\Xs^Q$ corresponds to the constraint that $\sum_{(s,t) \in E} R_{st}(x_{st})\leq \gamma$ with $R_{st}(x) = (x_{st} - \hat x_{st})^2 \sigma_{st}^{-2}$, and $x \in \boxset(0,1)$. By the monotonicity of $\Is(x,y)$, there is never incentive to reduce any
$x_{st}$ below $\hat x_{st}$, so we can replace
$\boxset(0,1)$ with $\boxset(\hat x,1)$. On this interval, each $R_{st}$ is
strictly increasing, and Theorem~\ref{thm:constrained-opt} applies.

For D-norm sets, we have $R_{st}(x_{st}) = (x_{st} - \hat x_{st}) / (u_{st} - \hat x_{st})$. Since each $R_{st}$ is monotone, Theorem~\ref{thm:constrained-opt} applies.

\paragraph{Runtime and Alternatives.}
Since the core algorithm is Frank-Wolfe, it is straightforward to show that Problem~\eqref{eq:rho-prob} can be solved to $\varepsilon$-suboptimality in time $O(\varepsilon^{-1} n^2\delta^{-3} \alpha^{-1} |T|^2 \log{n\delta^{-1}})$, where $\alpha$ is the minimum derivative of the functions $R_i$.
If $\rho^*$ has distinct elements separated by $\eta$, then choosing $\varepsilon = \eta^2 \alpha \delta / 8$ results in an exact solution to~\eqref{eq:regularized-prob-discrete} in time $O(\eta^{-2} n^2\delta^{-4} \alpha^{-2} |T|^2 \log{n\delta^{-1}})$.

Noting that $H^\delta + \lambda R^\delta$ is submodular for all $\lambda$, one could instead perform binary search over $\lambda$, each time converting the objective into a submodular \emph{set} function via Birkhoff's theorem and solving submodular minimization e.g.\ via one of the recent fast methods 
\citep{chakrabarty2017subquadratic,lee_faster_2015}. However, we are not aware of a practical implementation of the algorithm in~\citep{lee_faster_2015}. The algorithm in~\citep{chakrabarty2017subquadratic} yields a solution in expectation. This approach also requires care in the precision of the search over $\lambda$, whereas our approach searches directly over the $O(n\delta^{-1})$ elements of $\rho^*$.


\section{Experiments}
\label{sec:expts}
We evaluate our Robust Budget Allocation algorithm on both synthetic test data and a real-world bidding dataset from Yahoo! Webscope \cite{yahoo} to demonstrate that our method yields real improvements. For all experiments, we used Algorithm \ref{alg:subgrad} as the outer loop. For the inner submodular minimization step, we implemented the pairwise Frank-Wolfe algorithm of \cite{lacoste_global_2015}. In all cases, the feasible set of budgets $\Ys$ is $\{y \in \reals^S_+ : \sum_{s\in S} y(s) \leq C\}$ where the specific budget $C$ depends on the experiment. Our code is available at \url{git.io/vHXkO}.

\subsection{Synthetic}
On the synthetic data, we probe two questions: (1) how often does the distinctness condition of Theorem~\ref{thm:constrained-opt} hold, so that we are guaranteed an optimal solution; and (2) what is the gain of using a robust versus non-robust solution in an adversarial setting?
%
For both settings, 
we set $\abs S = 6$ and $\abs T = 2$ and discretize with $\delta=0.001$. We generated true probabilties $p_{st}$, created Beta posteriors, and built both Ellipsoidal uncertainty sets $\Xs^Q(\gamma)$ and D-norm sets $\Xs^D(\gamma)$.

\subsubsection{Optimality}
\label{sec:synthetic-optimality}
Theorem~\ref{thm:constrained-opt} and Remark~\ref{rem:rho-distinct} demand that the values $\rho^*_i(j)$ be distinct at our chosen Lagrange multiplier $\lambda^*$ and, under this condition, guarantee optimality. We illustrate this in four examples: for Ellipsoidal or a D-norm uncertainty set, and a total influence budget $C \in \{0.4, 4\}$.
Figure~\ref{fig:yahoo-convergence} shows all elements of $\rho^*$ in sorted order, as well as a horizontal line indicating our Lagrange multiplier $\lambda^*$ which serves as a threshold. Despite some plateaus, the entries $\rho^*_i(j)$ are distinct in most regimes, in particular around $\lambda^*$, the regime that is needed for our results.  Moreover, in practice (on the Yahoo data) we observe later in Figure~\ref{fig:yahoo-convergence} that both solution-dependent bounds from Remark~\ref{rem:soln-dependent-bounds} are very good, and all solutions are optimal within a very small gap.

\begin{figure}
    \begin{center}
        \includegraphics[width=5.42in]{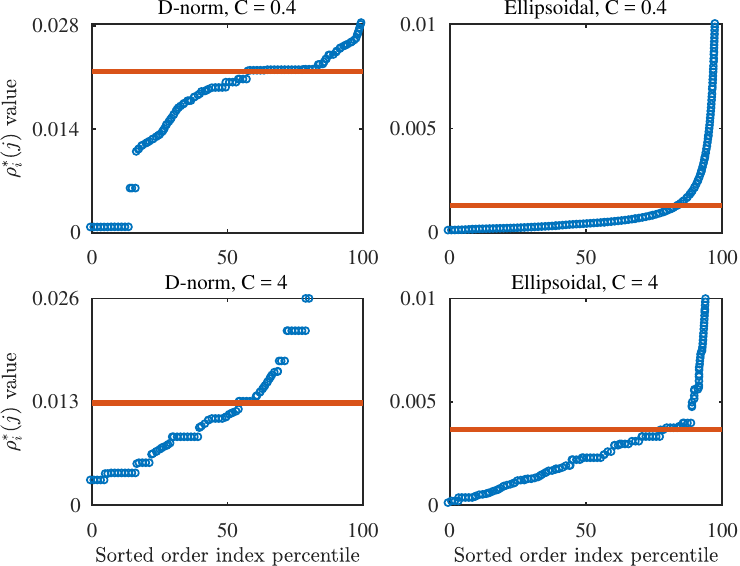}
    \end{center}
        \caption{Visualization of the sorted values of $\rho^*_i(j)$ (blue dots) with comparison to the particular Lagrange multiplier $\lambda^*$ (orange line). In most regimes there are no duplicate values, so that Theorem~\ref{thm:constrained-opt} applies. The theorem only needs distinctness at $\lambda^*$.}
        \label{fig:rho-plots}

\end{figure}

\subsubsection{Robustness and Quality}
Next, we probe the effect of a robust versus non-robust solution for different uncertainty sets and budgets $\gamma$ of the adversary. We compare our robust solution with using a point estimate for $x$, i.e., $y_\mathrm{nom} \in \argmax_{y\in\Ys} \Is(y; \hat x)$, treating estimates as ground truth, and the stochastic solution $y_\mathrm{expect} \in \argmax_{y\in\Ys} \E[\Is(y;X)]$ as per Section~\ref{sec:stoch-opt}. These two optimization problems were solved via standard first-order methods using TFOCS \cite{becker_templates_2011}.

Figure \ref{fig:synthetic-robust-influence} demonstrates that indeed, the alternative budgets are sensitive to the adversary and the robustly-chosen budget $y_\mathrm{robust}$ performs better, even in cases where the other budgets achieve zero influence. When the total budget $C$ is large, $y_\mathrm{expect}$ performs nearly as well as $y_\mathrm{robust}$, but when resources are scarce ($C$ is small) and the actual choice seems to matter more, $y_\mathrm{robust}$ performs far better.

\begin{figure}
\begin{center}
    \includegraphics[width=5.42in]{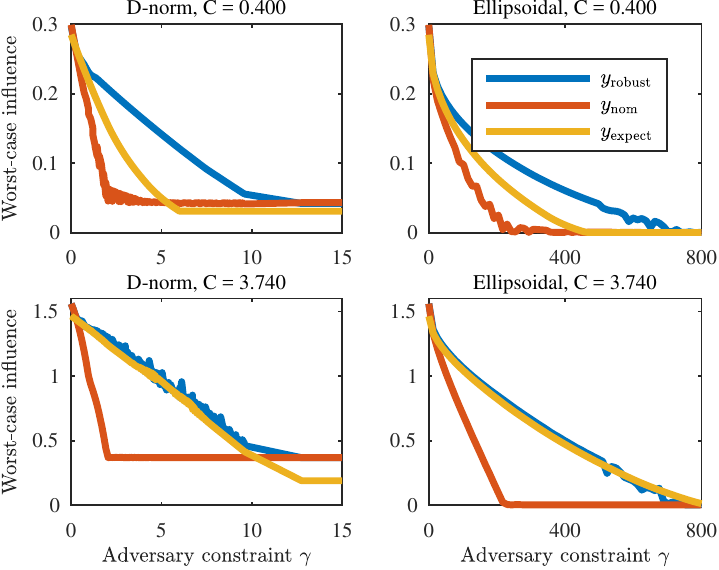}
\end{center}
\caption{Comparison of worst-case expected influences for D-norm uncertainty sets $\Xs^D(\gamma)$ (left)
  and ellipsoidal uncertainty sets $\Xs^Q(\gamma)$ (right), for different total budget bounds $C$.
  For any particular adversary budget $\gamma$, we compare $\min_{x\in\Xs(\gamma)} \Is(y;x)$ for each candidate allocation $y$.}
\label{fig:synthetic-robust-influence}
\end{figure}

\subsection{Yahoo! data}
To evaluate our method on real-world data, we formulate
a Budget Allocation instance on advertiser bidding data from
Yahoo! Webscope \citep{yahoo}. This dataset logs bids on 1000 different
phrases by advertising accounts. We map the phrases to channels $S$ and 
the accounts to customers $T$, with an edge between $s$ and $t$ if a
corresponding bid was made. For each pair $(s,t)$, we draw the associated
transmission probability $p_{st}$ uniformly from $[0,0.4]$. We bias these
towards zero because we expect people not to be easily influenced by
advertising in the real world. We then generate an estimate $\hat p_{st}$ and
build up a posterior by generating $n_{st}$ samples from
$\Ber(p_{st})$, where $n_{st}$ is the number of bids between
$s$ and $t$ in the dataset.

This transformation yields a bipartite graph with $\abs S = 1000$, $\abs T = 10475$, and 
more than 50,000 edges that we use for Budget Allocation.
In our experiments, the typical gap between
the naive $y_\mathrm{nom}$ and robust $y_\mathrm{robust}$ was 100-500 expected influenced people.
We plot convergence of the outer loop in Figure~\ref{fig:yahoo-convergence}, where we observe fast convergence of both
primal influence value and the dual bound.

\begin{figure}
\begin{center}
    \includegraphics[width=6in]{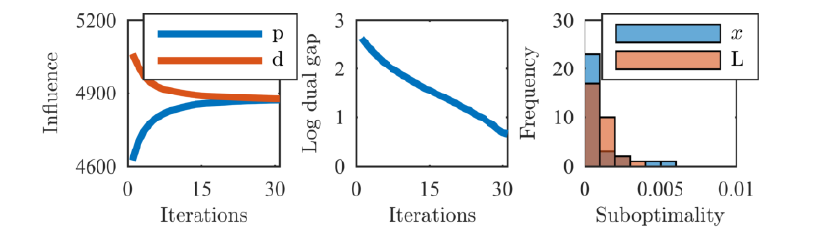}
\end{center}
\caption{Convergence properties of our algorithm on real data. In the first plot, `p' and `d' refer to primal and dual values, with dual gap shown on the second plot. The third plot demonstrates that the problem-dependent suboptimality bounds of Remark~\ref{rem:soln-dependent-bounds} ($x$ for $x(\lambda_+)$ and L for Lagrangian) are very small (good) for all inner iterations of this run.}
\label{fig:yahoo-convergence}
\end{figure}


\subsection{Comparison to first-order methods}
Given the success of first-order methods on nonconvex problems in practice, it is natural to compare these to our method for finding the worst-case vector $x$. On one of our Yahoo problem instances with D-norm uncertainty set, we compared our submodular minimization scheme to Frank-Wolfe with fixed stepsize as in \citep{lacoste-julien_convergence_2016}, implementing the linear oracle using MOSEK \citep{mosek}. Interestingly, from various initializations, Frank-Wolfe finds an optimal solution, as verified by comparing to the guaranteed solution of our algorithm. Note that, due to non-convexity, there are no formal guarantees for Frank-Wolfe to be optimal here, motivating the question of global convergence properties of Frank-Wolfe in the presence of submodularity.


It is important to note that there are many cases where first-order methods are inefficient or do not apply to our setup. These methods require either a projection oracle (PO) onto or linear optimization oracle (LO) over the feasible set $\Xs$ defined by $\ell$, $u$ and $R(x)$. The D-norm set admits a LO via linear programming, but we are not aware of any efficient LO for Ellipsoidal uncertainty, nor PO for either set, that does not require quadratic programming. Even more, our algorithm applies for nonconvex functions $R(x)$ which induce nonconvex feasible sets $\Xs$. Such nonconvex sets may not even admit a unique projection, while  our algorithm achieves provable solutions.


\begin{figure}
\begin{center}
    \includegraphics[width=3.17in]{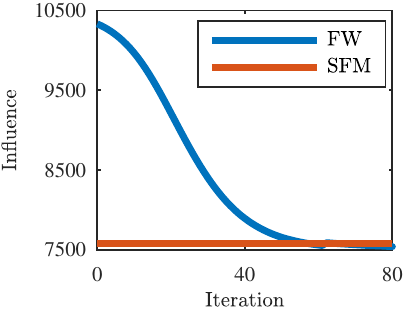}
\end{center}
\caption{Convergence properties of Frank-Wolfe (FW), versus the optimal value attained with our scheme (SFM).}
\label{fig:fw-compare}
\end{figure}



\section{Conclusion}
We address the issue of uncertain parameters (or, model mis-specification) in Budget Allocation or Bipartite Influence Maximization \citep{alon_optimizing_2012} from a robust optimization perspective. The resulting \emph{Robust Budget Allocation} is a nonconvex-concave saddle point problem. Although the inner optimization problem is nonconvex, we show how continuous submodularity can be leveraged to solve the problem to arbitrary accuracy $\ep$, as can be verified with the proposed bounds on the duality gap. In particular, our approach extends continuous submodular minimization methods \citep{bach_submodular_2015} to more general constraint sets, introducing a mechanism to solve a new class of constrained nonconvex optimization problems.
We confirm on synthetic and real data that our method finds high-quality solutions that are robust to parameters varying arbitrarily in an uncertainty set, and scales up to graphs with over 50,000 edges.

There are many compelling directions for further study. The uncertainty sets we use are standard in the robust optimization literature, but have not been applied to e.g. Robust Influence Maximization; it would be interesting to generalize our ideas to general graphs. Finally, despite the inherent nonconvexity of our problem, first-order methods are often able to find a globally optimal solution. Explaining this phenomenon requires further study of the geometry of constrained monotone submodular minimization.

\subsection*{Acknowledgements}
We thank the anonymous reviewers for their helpful suggestions. We also thank MIT Supercloud and the Lincoln Laboratory Supercomputing Center for providing computational resources.
This research was conducted with Government support under and awarded by DoD, Air Force Office of Scientific Research, National Defense Science and Engineering Graduate (NDSEG) Fellowship, 32 CFR 168a, and also supported by NSF CAREER award 1553284.

\bibliography{ref_short}
\bibliographystyle{icml2017nourl}

\appendix
\section{Worst-Case Approximation Ratio versus True Worst-Case}
\label{sec:appendix-ratio}
Consider the function $f(x;\theta)$ defined on $\{0,1\} \times \{0,1\}$,
with values given by: 
\begin{equation}
    f(x;0) = \begin{cases}
        1 & x=0 \\
        0.6 & x=1,
    \end{cases} 
    \quad f(x;1) = \begin{cases}
        1 & x=0 \\
        2 & x=1.
    \end{cases}
\end{equation}
We wish to choose $x$ to maximize $f(x;\theta)$ robustly with respect to
adversarial choices of $\theta$.
If $\theta$ were fixed, we could directly choose $x_\theta^*$ to maximize $f(x;\theta)$. In particular, $x^*_0 = 0$ and $x^*_1 = 1$.
Of course, we want to deal with worst-case $\theta$. One option is to maximize the worst-case approximation ratio:
\begin{equation}
    \max_x \min_\theta \frac{f(x;\theta)}{f(x^*_\theta;\theta)}.
\end{equation}
One can verify that the best $x$ according to this criterion is $x=1$, with worst-case approximation ratio 0.6 and worst-case function value 0.6.
In this paper, we optimize the worst-case of the actual function value:
\begin{equation}
    \max_x \min_\theta f(x;\theta).
\end{equation}
This criterion will select $x=0$, which has a worse worst-case approximation
ratio of 0.5, but actually guarantees a function value of 1, significantly
better than the 0.6 achieved by the other formulation of robustness.


\section{DR-submodularity and $L^\natural$-convexity}
\label{sec:appendix-dr}
A function is $L^\natural$-convex if it satisfies a discrete version of midpoint convexity, i.e. for all $x,y$ it holds that
\begin{equation}
f(x) + f(y) \geq f\left(\left\lceil \frac{x+y}{2}\right\rceil\right) + f\left(\left\lfloor \frac{x+y}{2}\right\rfloor \right).
\end{equation}

\begin{remark}
    An $L^\natural$-convex function need not be DR-submodular, and vice-versa.
    Hence algorithms for optimizing one type may not apply for the other.
\end{remark}
\begin{proof}
Consider $f_1(x_1,x_2) = -x_1^2 - 2x_1 x_2$ and $f_2(x_1,x_2) = x_1^2 + x_2^2$, both defined on $\{0,1,2\} \times \{0,1,2\}$. The function $f_1$ is DR-submodular but violates discrete midpoint convexity for the pair of points $(0,0)$ and $(2,2)$, while $f_2$ is $L^\natural$-convex but does not have diminishing returns in either dimension.
\end{proof}

Intuitively-speaking, $L^\natural$-convex functions look like discretizations of convex functions. The continuous objective function $\Is(x,y)$ we consider need not be convex, hence its discretization need not be $L^\natural$-convex, and we cannot use those tools. However, in some regimes (namely if each $y(s) \in \{0\} \cup [1,\infty)$), it happens that $\Is(x,y)$ is DR-submodular in $x$.

\section{Constrained Continuous Submodular Function Minimization}
\matt{todo: modify notation of $\rho_{i x_i}$ vs $\rho_i(x_i)$ to be consistent with the earlier section}
Define $\reals_\downarrow^n$ to be the set of vectors $\rho$ in $\reals^n$
which are monotone nonincreasing, i.e. $\rho(1) \geq \rho(2) \geq \dots \geq
\rho(n)$. As in the main text, define $[k] = \{0,1,\dots,k-1\}$.
One of the key results from \cite{bach_submodular_2015} is that an arbitrary submodular function
$H(x)$ defined on $\prod_{i=1}^n [k_i]$ can be extended to a particular convex function $h_\downarrow(\rho)$ so that
\begin{equation}
    \begin{array}{ll}
        \text{minimize} & H(x) \\
        \text{s.t.} & x \in \prod_{i=1}^n [k_i]
    \end{array}
    \Leftrightarrow
    \begin{array}{ll}
        \text{minimize} & h_\downarrow(\rho) \\
        \text{s.t.} & \rho \in \prod_{i=1}^n \reals_\downarrow^{k_i - 1}.
    \end{array}
\end{equation}
Moreover, Theorem 4 from \cite{bach_submodular_2015} states that, if $a_{i y_i}$ are strictly convex functions
for all $i=1,\dots,n$ and each $y_i \in [k_i]$, then the two problems
\begin{equation}
    \label{eq:submod-plus-regularizer}
    \begin{array}{ll}
        \text{minimize} & H(x) + \sum_{i=1}^n \sum_{y_i=1}^{x_i} a^\prime_{i y_i} (\lambda) \\
        \text{s.t.} & x \in \prod_{i=1}^n [k_i].
    \end{array}
\end{equation}
and
\begin{equation}
    \label{eq:rho-plus-convex}
    \begin{array}{ll}
        \text{minimize} & h_\downarrow(\rho) + \sum_{i=1}^n \sum_{x_i=1}^{k_i-1} a_{i x_i}[\rho_i(x_i)] \\
        \text{s.t.} & \rho \in \prod_{i=1}^n \reals_\downarrow^{k_i - 1}
    \end{array}
\end{equation}
are equivalent. In particular, one recovers a solution to Problem~\eqref{eq:submod-plus-regularizer} for any $\lambda$ just as
alluded to in Lemma~\ref{lemma:prob-equiv}: 
find $\rho^*$ which solves Problem~\eqref{eq:rho-plus-convex} and, for each component $i$, choose $x_i$ to be
the maximal value for which $\rho^*_i(x_i) \geq \lambda$.

\label{sec:appendix-ccsfm}
\subsection{Proof of Lemma~\ref{lemma:prob-equiv}}
\begin{proof}
The discretized form of the regularizer $R^\delta$ is also separable and can be
written $R^\delta(x) = \sum_{i=1}^n R^\delta_i(x)$.
For each $i=1,\dots,n$ and each $y_i \in [k_i]$ with $y_i \geq 1$, define $a_{i
y_i}(t) = \frac12 t^2 \cdot [R^\delta_i(y_i) - R^\delta_i(y_i - 1)]$, so that
$a^\prime_{i y_i}(t) = t \cdot [R^\delta_i(y_i) - R^\delta_i(y_i - 1)]$. Since we
assumed $R(x)$ is strictly increasing, the coefficient of $t^2$ in each $a_{i
y_i}(t)$ is strictly positive, so that each $a_{i y_i}(t)$ is strictly convex.
Then,
\begin{align}
    \lambda R^\delta_i(x_i) &= \lambda \cdot \left[ R^\delta_i(0) + \sum_{y_i=1}^{x_i} \left( R^\delta_i(y_i) - R^\delta_i(y_i - 1)\right)\right] \\
&= \lambda R^\delta_i(0) + \sum_{y_i=1}^{x_i} a^\prime_{i y_i}(\lambda),
\end{align}
so that the discretized version of the minimization problem can be written as
\begin{equation}
    \begin{array}{ll}
        \text{minimize} & H^\delta(x) + \lambda R^\delta(0) + \sum_{i=1}^n \sum_{y_i=1}^{x_i} a^\prime_{i y_i} (\lambda) \\
        \text{s.t.} & x \in \prod_{i=1}^n [k_i].
    \end{array}
\end{equation}
Since the term $R^\delta(0)$ does not depend on the variable $x$, this minimization is equivalent to
\begin{equation}
    \begin{array}{ll}
        \text{minimize} & H^\delta(x) + \sum_{i=1}^n \sum_{y_i=1}^{x_i} a^\prime_{i y_i} (\lambda) \\
        \text{s.t.} & x \in \prod_{i=1}^n [k_i].
    \end{array}
\end{equation}
This problem is in the precise form where we can apply the preceding equivalence result between
Problems~\eqref{eq:submod-plus-regularizer} and \eqref{eq:rho-plus-convex}, so we are done.
\end{proof}

\subsection{Proof of Theorem~\ref{thm:constrained-opt}}
\begin{proof}
The general idea of this proof is to first show that the integer-valued
point $x_d^*$ which 
solves
\begin{equation*}
    x_d^* \in \argmin_{x \in \prod_{i=1}^n [k_i] : R^\delta(x) \leq B} H^\delta(x)
\end{equation*}
is also nearly a minimizer of the continuous version of the problem, due
to the fineness of the discretization. Then, we show that the solutions
traced out by $x(\lambda)$ get very close to $x_d^*$. These two results
are simply combined via the triangle inequality.

\subsubsection{Continuous and Discrete Problems}
We begin by proving that 
\begin{equation}
    \label{eq:ctns-discrete-ineq}
    H^\delta(x_d^*)
    \leq G\delta + \min_{x \in \Xs : R(x) \leq B} H(x).
\end{equation}
Consider $x^* \in \arg\min_{x \in \Xs : R(x) \leq B} H(x)$. If $x^*$
corresponds to an integral point in the discretized domain, then $H(x^*) =
H^\delta(x_d^*)$ and we are done. Else, $x^*$ has at least one
non-integral coordinate. By rounding coordinatewise, we can construct
a set $X = \{x_1, \dots, x_m\} \subseteq \prod_{i=1}^n [k_i]$ so that
$x^* \in \conv(\{A(x_1),\dots,A(x_m)\}$. By monotonicity, there must be some $x_i \in X$
with $R^\delta(x_i) \leq B$, i.e. $A(x_i)$ is feasible for the original
continuous problem. By construction, since the discretization given by $A$ is
$\delta$-fine, we must have $\lVert x^* - A(x_i) \rVert_\infty \leq \delta$.
Applying the Lipschitz property of $H$ and the optimality of $x^*$, we have
\[
    G\delta \geq H(A(x_i)) - H(x^*) = H^\delta(x_i) - H(x^*) \geq H^\delta(x_d^*) - H(x^*),
\]
from which~\eqref{eq:ctns-discrete-ineq} follows.

\subsubsection{Discrete and Parameterized Discrete Problems}
Define $\lambda_-$ and $\lambda_+$ by
\begin{align*}
    \lambda_- &\in \argmin_{\lambda \geq 0 : R^\delta(x(\lambda)) \leq B} H^\delta(x(\lambda)) \quad \text{and} \\
    \lambda_+ &\in \argmax_{\lambda \geq 0 : R^\delta(x(\lambda)) \geq B} H^\delta(x(\lambda)).
\end{align*}
The next step in proving our suboptimality bound is to prove that 
\begin{equation}
    H^\delta(x(\lambda_+)) \leq H^\delta(x_d^*) \leq H^\delta(x(\lambda_-)),
\end{equation}
from which it will follow that 
\begin{equation*}
    H^\delta(x(\lambda_-)) \leq G\delta + H^\delta(x_d^*).
\end{equation*}

We begin by stating the min-max inequality, i.e. weak duality:
\begin{align}
    \min_{x \in \prod_{i=1}^n [k_i] : R^\delta(x) \leq B} H^\delta(x)
    &= \min_{x \in \prod_{i=1}^n [k_i]} \max_{\lambda \geq 0} \left\{ H^\delta(x) + \lambda (R^\delta(x) - B) \right\} \\
    &\geq \max_{\lambda \geq 0} \min_{x \in \prod_{i=1}^n [k_i]} \left\{ H^\delta(x) + \lambda (R^\delta(x) - B) \right\} \\
    &= \max_{\lambda \geq 0} \left\{ H^\delta(x(\lambda)) + \lambda (R^\delta(x(\lambda)) - B) \right\} \\
    &\geq \max_{\lambda \geq 0 : R^\delta(x(\lambda)) \geq B} \left\{ H^\delta(x(\lambda)) + \lambda (R^\delta(x(\lambda)) - B) \right\} \\
    &\geq \max_{\lambda \geq 0 : R^\delta(x(\lambda)) \geq B} H^\delta(x(\lambda)) \\
    &= H^\delta(x(\lambda_+)).
\end{align}
We can also bound the optimal value of $H^\delta(x_d^*)$ from the other
side:
\begin{equation}
    H^\delta(x_d^*)
    = \min_{x \in \prod_{i=1}^n [k_i] : R^\delta(x) \leq B} H^\delta(x)
    \leq \min_{\lambda \geq 0 : R^\delta(x(\lambda)) \leq B} H^\delta(x)
    = H^\delta(x(\lambda_-)) 
\end{equation}
because the set of $x(\lambda)$ parameterized by $\lambda$ is a subset of the
full set $\{ x \in \prod_{i=1}^n [k_i] : R^\delta(x) \leq B \}$.

We have now bounded the optimal value of $H^\delta(x_d^*)$ on either side
by optimization problems where we seek an optimal $\lambda \geq 0$ for the
parameterization $x(\lambda)$:
\begin{equation}
    H^\delta(x(\lambda_+))
    \leq
    H^\delta(x_d^*)
    \leq
     H^\delta(x(\lambda_-)).
\end{equation}
Recall that $x(\lambda)$ comes from thresholding the values of $\rho^*$
by $\lambda$, and that we assume that the elements of $\rho^*$ are
unique. Hence, as we increase $\lambda$, the components of $x$ decrease by
1 each time. Combining this with the strict monotonicity of $R$, we see that
$\lVert x(\lambda_+) - x(\lambda_-) \rVert_\infty \leq 1$. By the Lipschitz
properties of $H^\delta$, it follows that 
$\abs{H^\delta(x(\lambda_+)) - H^\delta(x(\lambda_-))} \leq G\delta$.
Since $H^\delta(x_d^*)$ lies in the interval between $H^\delta(x(\lambda_+))$ and $H^\delta(x(\lambda_-))$, it follows that 
$\abs{H^\delta(x_d^*) - H^\delta(x(\lambda_-))} \leq G\delta$.
\end{proof}

%
\subsection{Proof of Remark~\ref{rem:soln-dependent-bounds}}
Define $\lambda^* = \lambda_-$ as in the previous section, so that $x' = A(x(\lambda^*)$. The $x(\lambda_+)$ bound is a simple consequence from the above result that
\[
    H^\delta(x(\lambda_+)) \leq H^\delta(x_d^*) \leq H^\delta(x(\lambda_-)) = H(x').
\]

As for the Lagrangian bound, since $x(\lambda^*)$ is a minimizer for the regularized function $H^\delta(x) + \lambda^* (R^\delta(x) - B)$, it follows that
\begin{equation}
    H^\delta(x(\lambda^*)) + \lambda^* (R^\delta(x(\lambda^*)) - B) \leq H^\delta(x_d^*) + \lambda^* (R^\delta(x_d^*) - B).
\end{equation}
Rearranging, and observing that $R^\delta(x_d^*) \leq B$ because $x_d^*$ is feasible, it holds that
\begin{equation}
    H(x') = H^\delta(x(\lambda^*)) \leq H^\delta(x_d^*) + \lambda^* (R^\delta(x_d^*) - R^\delta(x(\lambda^*)))
    \leq H^\delta(x_d^*) + \lambda^* (B - R(x')).
\end{equation}

One can also combine either of these bounds with the result from the proof of Theorem~\ref{thm:constrained-opt} that $H^\delta(x_d^*) \leq G\delta + H(x^*)$ yielding e.g.
\[
    H(x') \leq G\delta + \lambda^*(B - R(x')) + H^\delta(x_d^*).
\]

\subsection{Solving the Optimization Problem}
Now that we have proven equivalence results between the constrained problem we want to solve
and the convex problem~\eqref{eq:rho-plus-convex}, we need to actually solve the convex problem.
At the beginning of Section 5.2 in \cite{bach_submodular_2015}, it is stated
that this surrogate problem can optimized via the Frank-Wolfe method and its
variants, but only the the version of Problem~\eqref{eq:rho-plus-convex} without the extra functions $a_{i x_i}$ is elaborated upon. Here we detail how Frank-Wolfe algorithms can be used to solve the
more general parametric regularized problem. Our aim is to spell out very
clearly the applicability of Frank-Wolfe to this problem, for the ease of
future practitioners.

\citet{bach_submodular_2015} notes that by duality, Problem~\eqref{eq:rho-plus-convex} is equivalent to:
\begin{align*}
    \min_{\rho \in \prod_{i=1}^n \reals_\downarrow^{k_i - 1}}
    h_\downarrow(\rho) - H(0) + \sum_{i=1}^n \sum_{x_i=1}^{k_i-1} a_{i x_i}[\rho_i(x_i)] 
    &= \min_{\rho \in \prod_{i=1}^n \reals_\downarrow^{k_i - 1}}
    \max_{w \in B(H)} \langle \rho, w \rangle + \sum_{i=1}^n \sum_{x_i=1}^{k_i-1} a_{i x_i}[\rho_i(x_i)] \\
    &= 
    \max_{w \in B(H)} \left\{
    \min_{\rho \in \prod_{i=1}^n \reals_\downarrow^{k_i - 1}}
\langle \rho, w \rangle + \sum_{i=1}^n \sum_{x_i=1}^{k_i-1} a_{i x_i}[\rho_i(x_i)] \right\} \\
    &:= \max_{w \in B(H)} f(w).
\end{align*}
Here, the base polytope $B(H)$ happens to be the convex hull of all vectors $w$ 
which could be output by the greedy algorithm in \cite{bach_submodular_2015}.

It is the dual problem, where we maximize over $w$, which is amenable to
Frank-Wolfe. For Frank-Wolfe methods, we need two oracles: an oracle which,
given $w$, returns $\nabla f(w)$; and an oracle which, given $\nabla f(w)$,
produces a point $s$ which solves the linear optimization problem $\max_{s \in
B(H)} \langle s, \nabla f(w) \rangle$.

Per \citet{bach_submodular_2015}, an optimizer of the linear
problem can be computed directly from the greedy algorithm. \matt{elaborate}
For the gradient oracle, recall that we can find a subgradient of $g(x) =
\min_y h(x,y)$ at the point $x_0$ by finding $y(x_0)$ which is optimal for the
inner problem, and then computing $\nabla_x h(x,y(x_0))$. Moreover, if such
$y(x_0)$ is the unique optimizer, then the resulting vector is indeed the
\emph{gradient} of $g(x)$ at $x_0$. Hence, in our case, it suffices to first find $\rho(w)$ which solves the inner problem, and then $\nabla f(w)$ is simply $\rho(w)$ because the inner function is linear in $w$. Since each function $a_{i x_i}$ is strictly convex, the minimizer $\rho(w)$ is unique, confirming that we indeed get a gradient of $f$, and that $f$ is differentiable. 

Of course, we still need to compute the minimizer $\rho(w)$. For a given $w$, we want to solve
\begin{align*}
    \min_{\rho \in \prod_{i=1}^n \reals_\downarrow^{k_i - 1}}
\langle \rho, w \rangle + \sum_{i=1}^n \sum_{x_i=1}^{k_i-1} a_{i x_i}[\rho_i(x_i)] 
\end{align*}
There are no constraints coupling the vectors $\rho_i$, and the objective is similarly separable, so we can independently solve $n$ problems of the form
\begin{align*}
    \min_{\rho \in \reals_\downarrow^{k - 1}}
\langle \rho, w \rangle + \sum_{j=1}^{k-1} a_{j}(\rho_j).
\end{align*}
Recall that each function $a_{i y_i}(t)$ takes the form $\frac12 t^2 r_{i y_i}
$ for some $r_{i y_i} > 0$. Let $D = \diag(r)$, the $(k-1)\times(k-1)$ matrix with diagonal entries $r_j$. Our problem can then be written as
\begin{align*}
    \min_{\rho \in \reals_\downarrow^{k - 1}}
\langle \rho, w \rangle + \frac12 \sum_{j=1}^{k-1} r_j \rho_j^2
    &=
    \min_{\rho \in \reals_\downarrow^{k - 1}}
    \langle \rho, w \rangle + \frac12 \langle D \rho, \; \rho \rangle \\
    &=
    \min_{\rho \in \reals_\downarrow^{k - 1}}
    \langle D^{1/2}\rho, \; D^{-1/2} w \rangle + \frac12 \langle D^{1/2} \rho, \; D^{1/2}\rho \rangle. 
\end{align*}
Completing the square, the above problem is equivalent to
\begin{align*}
    \min_{\rho \in \reals_\downarrow^{k - 1}}
    \lVert D^{1/2} \rho + D^{-1/2} w \rVert_2^2
    &= 
    \min_{\rho \in \reals_\downarrow^{k - 1}}
    \sum_{j=1}^{k-1} (r_j^{1/2} \rho_j + r_j^{-1/2} w_j)^2 \\
    &= 
    \min_{\rho \in \reals_\downarrow^{k - 1}}
    \sum_{j=1}^{k-1} r_j (\rho_j + r_j^{-1} w_j)^2.
\end{align*}
This last expression is precisely the problem which is called weighted
isotonic regression: we are fitting $\rho$ to $\diag(r^{-1}) w$, with weights
$r$, subject to a monotonicity constraint. Weighted isotonic regression is
solved efficiently via the Pool Adjacent Violators algorithm of
\cite{best_active_1990}.

\subsection{Runtime}
Frank-Wolfe returns an $\ep$-suboptimal solution in $O(\ep^{-1} D^2 L)$ iterations, where $D$ is the diameter of the feasible region, and $L$ is the Lipschitz constant for the gradient of the objective~\citep{jaggi_revisiting_2013}. Our optimization problem is $\max_{w\in B(H)} f(w)$ as defined in the previous section. Each $w \in B(H)$ has $O(n\delta^{-1})$ coordinates of the form $H^\delta(x+e_i)-H^\delta(x)$. Since $H^\delta$ is an expected influence in the range $[0,T]$, we can bound the magnitude of each coordinate of $w$ by $T$ and hence $D^2$ by $O(n\delta^{-1} T^2)$. If $\alpha$ is the minimum derivative of the functions $R_i$, then the smallest coefficient of the functions $a_{ix_i}(t)$ is bounded below by $\alpha\delta$. Hence the objective is the conjugate of an $\alpha\delta$-strongly convex function, and therefore has $\alpha^{-1}\delta^{-1}$-Lipschitz gradient.
Combining these, we arrive at the $O(\ep^{-1} n\delta^{-2} \alpha^{-1} T^2)$ iteration bound. The most expensive step in each iteration is computing the subgradient, which requires sorting the $O(n\delta^{-1})$ elements of $\rho$ in time $O(n\delta^{-1} \log{n\delta^{-1}} )$. Hence the total runtime of Frank-Wolfe is $O(\ep^{-1} n^2\delta^{-3} \alpha^{-1} T^2 \log{n\delta^{-1}})$. 

As specified in the main text, relating an approximate solution of~\eqref{eq:rho-prob} to a solution of~\eqref{eq:regularized-prob-discrete} is nontrivial. Assume $\rho^*$ has distinct elements separated by $\eta$, and chose $\ep$ to be less than $\eta^2 \alpha \delta / 8$. If $\rho$ is $\ep$-suboptimal, then by $\alpha\delta$-strong convexity we must have $\lVert \rho - \rho^* \rVert_2 < \eta/2$, and therefore $\lVert \rho - \rho^* \rVert_\infty < \eta/2$. Since the smallest consecutive gap between elements of $\rho^*$ is $\eta$, this implies that $\rho$ and $\rho^*$ have the same ordering, and therefore admit the same solution $x$ after thresholding. Accounting for this choice in $\ep$, we have an exact solution to~\eqref{eq:regularized-prob-discrete} in total runtime of $O(\eta^{-2} n^2\delta^{-4} \alpha^{-2} T^2 \log{n\delta^{-1}})$.

\section{Expectation and Variance of the Influence Function}
We wish to study the influence $\Is(X,y)$, its expectation and its variance
as a function of $y$. By definition, the influence function is given by
\begin{equation}
    \Is(X,y) = \sum_{t\in T} \left( 1 - \prod_{(s,t)\in E} X_{st}^{y(s)} \right).
\end{equation}
Before we prove the stated results, we will simplify the functions involved.

Maximizing $\Is(X,y)$ is equivalent to minimizing the function
\begin{equation}
    \sum_{t\in T}  \prod_{(s,t)\in E} X_{st}^{y(s)} 
\end{equation}
and vice-versa. The particular properties we are interested in, namely
convexity and submodularity, are preserved under sums. Moreover, expectation
is linear and variances add, so for our purposes we can focus on only one term
of the above sum. After reindexing in terms of $i=1,\dots,n$ instead of
$(s,t)\in E$, we are left studying functions of the form 
\begin{equation}
    f(y) = \prod_{i=1}^n X_i^{y_i}.
\end{equation}
If $f(y)$ is always convex (or supermodular), then $\Is(X,y)$ is always concave (submodular) in $y$, and similarly for their expectations and variances.

\paragraph{Expectation} 
By independence, 
\begin{equation}
    \E[f(y)] = \prod_{i=1}^n \E[X_i^{y_i}].
\end{equation}
Suppose that each $X_i \sim \betadist(\alpha_i,\beta_i)$, so that
\begin{align}
    \E[X_i^{y_i}] &= \frac{\Gamma(\alpha_i + \beta_i) \Gamma(\alpha_i + y_i)}{\Gamma(\alpha_i + \beta_i + y_i)} \\
    &= \frac{\Gamma(\alpha_i + \beta_i) \Gamma(\alpha_i + y_i)\Gamma(\beta_i)}{\Gamma(\alpha_i + \beta_i + y_i)\Gamma(\beta_i)} \\
    &= \frac{B(\alpha_i + y_i, \beta_i)}{B(\alpha_i, \beta_i)}.
\end{align}
Then,
\begin{equation}
    \E[f(y)] = \prod_{i=1}^n \frac{B(\alpha_i + y_i, \beta_i)}{B(\alpha_i, \beta_i)}
    \propto 
    \prod_{i=1}^n B(\alpha_i + y_i, \beta_i),
\end{equation}
where by $\propto$ we mean that the product of the denominators is a positive
constant, dependent on the problem data but independent of $y$. 

\paragraph{Variance}
The variance of $f(y)$ can be written as
\begin{align}
    \Var\left[\prod_{i=1}^n X_i^{y_i}\right]
    &= \E\left[\prod_{i=1}^n X_i^{2y_i}\right] - \E\left[\prod_{i=1}^n X_i^{y_i}\right]^2 \\
    &= \prod_{i=1}^n \E\left[X_i^{2y_i}\right] - \prod_{i=1}^n \E\left[X_i^{y_i}\right]^2 \\
    &= \prod_{i=1}^n \frac{B(\alpha_{i} + 2y_i, \beta_{i})}{B(\alpha_{i}, \beta_{i})}
    - \prod_{i=1}^n \left(\frac{B(\alpha_{i} + y_i, \beta_{i})}{B(\alpha_{i}, \beta_{i})}\right)^2.
    \label{eq:influence-variance}
\end{align}

\subsection{Gradient of Expected Influence}
Recall the identity 
\begin{equation}
    \frac{\partial}{\partial a} B(a,b) = B(a,b) (\psi(a) - \psi(a+b)),
\end{equation}
where $\psi$ is the digamma function.
We can then compute each component of the gradient of $\E[f(y)]$:
\begin{align}
    \frac{\partial}{\partial y_i} \left( \E[f(y)] \right)
    &= \prod_{i=1}^n \frac{1}{B(\alpha_i,\beta_i)} \cdot \prod_{j \not= i} B(\alpha_j+y_j,\beta_j) \cdot \frac{\partial}{\partial y_i} \left( B(\alpha_i+y_i,\beta_i) \right) \\
    &= \prod_{i=1}^n \frac{1}{B(\alpha_i,\beta_i)} \cdot \prod_{j \not= i} B(\alpha_j+y_j,\beta_j) \cdot B(\alpha_i + y_i, \beta_i) \cdot (\psi(\alpha_i+y_i) - \psi(\alpha_i+y_i+\beta_i)) \\
    &= \E[f(y)] \cdot (\psi(\alpha_i+y_i) - \psi(\alpha_i+y_i+\beta_i)).
\end{align}

\subsection{Counterexample for Fact~\ref{fact:variance-bad}}
We give a specific choice of parameters $n,\alpha_i,\beta_i$ and $y_i$ for
which the resulting function $\sqrt{\Var(f(y))}$ is non-convex, non-concave,
non-submodular and non-supermodular for various points $y \in \reals^n_+$. For
the case $T=1$, the function $1-f(y)$ is a valid influence function, so we
have a valid counterexample for $\sqrt{\Var(\Is(X,y))}$.

Consider the case $n=2$, with $\alpha_1=\alpha_2=1$ and $\beta_1=\beta_2=1$.
This corresponds to the Budget Allocation problem where we have two sources each with an edge to one customer, and we have only our prior (i.e. no data) on either of the edge probabilities.
Using equation~\eqref{eq:influence-variance}, we can directly compute the
Hessian of $\sqrt{\Var(f(y))}$ at any point $y$, e.g. using Mathematica. In
particular, for $y_1=y_2=1$, the Hessian has a positive and a negative
eigenvalue, so $\sqrt{\Var(f(y))}$ is neither convex nor concave at this
point. Also for $y_1=y_2=1$, the off-diagonal element is negative, so
$\sqrt{\Var(f(y))}$ is not supermodular over all of $\reals^2_+$. However, for
$y_1=y_2=3$, the off-diagonal element is positive, so our function is also not
submodular.

\end{document}